\def\year{2017}\relax
\documentclass[letterpaper]{article}

\usepackage{aaai17}
\usepackage{times,helvet,courier}
\usepackage{amsmath,amsthm,amsfonts,amssymb}
\usepackage{algorithm}
\usepackage[group-separator={,}]{siunitx}
\usepackage{graphicx}
\usepackage{tikz}
\usetikzlibrary{arrows,shapes,backgrounds}
\tikzset{>=latex}
\usepackage[pageanchor=true,plainpages=false, pdfpagelabels]{hyperref}   
\hypersetup{
    hidelinks=true,
}
 
\usepackage{subfig}
\usepackage{enumitem}

\usepackage{bm}
\usepackage{import}

\newcommand{\vect}[1]{\boldsymbol{\mathbf{#1}}}
\newcommand{\norm}[1]{\left\lVert#1\right\rVert}

\newtheorem{theorem}{Theorem}
\newtheorem{lemma}{Lemma}          

\newcommand{\THRESHOLD}{f_{\textsf{thresh}}}
\newcommand{\THRESHOLDt}[1]{\THRESHOLD^{{#1}}}

\newcommand{\fthresh}[1]{f_{\textsf{thresh}}^{#1}}
\newcommand{\abs}[1]{\left\lvert#1\right\rvert}
\newcommand{\suchthat}{: }
\newcommand{\ConflictSet}{Q}

\newcommand{\Z}{\mathbb{Z}}

\newcommand{\LCYCLECOVER}{$L$\textsc{-Cycle-Cover}}
\newcommand{\LFLIPCOVER}{$L$\textsc{-Flip-and-Cycle-Cover}}
\newcommand{\CYCLEFEAS}{f_{C}}

\newcommand{\FEASSET}{\mathcal{T}}
\newcommand{\cardinality}[1]{\lvert #1 \rvert}
\newcommand{\type}{\theta}
\newcommand{\typeSet}{\Theta}
\newcommand{\typeVec}{\vect{\type}}

\title{Small Representations of Big Kidney Exchange Graphs}
\author{John P. Dickerson$^{*,+}$ \and Aleksandr M. Kazachkov$^+$ \and Ariel Procaccia$^+$ \and Tuomas Sandholm$^+$\\ 
University of Maryland$^*$ \ \ \ \ \ Carnegie Mellon University$^+$\\
\small\texttt{john@cs.umd.edu, \{arielpro,sandholm\}@cs.cmu.edu, akazachk@cmu.edu}
}

\begin{document}
\maketitle

\begin{abstract}
Kidney exchanges are organized markets where patients swap willing but incompatible donors. In the last decade, kidney exchanges grew from small and regional to large and national---and soon, international.  This growth results in more lives saved, but exacerbates the empirical hardness of the $\mathcal{NP}$-complete problem of optimally matching patients to donors.  State-of-the-art matching engines use integer programming techniques to clear fielded kidney exchanges, but these methods must be tailored to specific models and objective functions, and may fail to scale to larger exchanges. In this paper, we observe that if the kidney exchange compatibility graph can be encoded by a constant number of patient and donor attributes, the clearing problem is solvable in polynomial time. We give necessary and sufficient conditions for losslessly shrinking the representation of an arbitrary compatibility graph. Then, using real compatibility graphs from the UNOS US-wide kidney exchange, we show how many attributes are needed to encode real graphs. The experiments show that, indeed, small numbers of attributes suffice.
\end{abstract}

\section{Introduction}\label{sec:intro}

There are over \num{100000} needy patients waiting for a kidney transplant in the United States, with similar---and increasing---demand worldwide.\footnote{{\tiny\texttt{https://optn.transplant.hrsa.gov/converge/data/}}}  Complementing potential cadaveric transplantation via the deceased donor waiting list, a recent innovation---kidney exchange~\cite{Rapaport86:Case,Roth04:Kidney}---allows patients with willing \emph{living} donors to participate in cyclic donor swaps or altruist-initiated donation chains to receive a life-saving organ.  Kidney exchange now accounts for over 10\% of living donation in the US, with that percentage increasing annually.

In reality, participating patients and donors are endowed with a set of attributes: blood type, tissue type, age, insurance, home transplant center, willingness to travel, and myriad other measurements of health, personal preference, and logistical constraints.  While some of these features can, at a cost, be temporarily or permanently changed, the attributes determine the feasibility of a potential donation from each donor to each patient.  
For example, a donor with blood type AB can only give to a patient with that blood type.

A central aspect of kidney exchange is the \emph{clearing problem}, that is, determining the ``best'' set of cyclic and chain-based swaps to perform in a given \emph{compatibility graph}, which consists of all participating patients, donors, and their potential feasible transactions.  For even simple (but realistic) models of kidney exchange, the clearing problem is $\mathcal{NP}$-hard~\cite{Abraham07:Clearing,Biro09:Maximum} and also extremely difficult to solve in practice~\cite{Glorie14:Kidney,Anderson15:Finding,Dickerson16:Position}.

In this paper, we tackle the complexity of the clearing problem via the introduction of a novel model for kidney exchange that explicitly takes into account all attributes of the participating patients and donors. Under the assumption that real kidney exchange graphs can be represented using 
a constant number of attributes, we show that our model permits polynomial-time solutions to central $\mathcal{NP}$-hard problems in general kidney exchange.
Inspired by classical results from intersection graph theory, we give conditions on the representation of arbitrary graphs in our model, and generalize to the case where participants are allowed to have a thresholded number of negative interactions between attributes.
Noting that real-life kidney exchange graphs are \emph{not} arbitrary, we show on actual data from the United Network for Organ Sharing (UNOS) US-wide kidney exchange that our model permits lossless representation of true graphs with far fewer attributes than the worst-case theoretical results require. 


\section{A New Model for Kidney Exchange}\label{sec:model}

In this section, we formalize our model of kidney exchange.  We prove that under this model certain well-known $\mathcal{NP}$-hard problems in general kidney exchange are solvable in polynomial time.  We also show that, given a compatibility graph, determining the best set of attributes to change (at some cost) is solvable in polynomial time.

\subsection{Notation \& Preliminaries}
A kidney exchange can be represented by a directed \emph{compatibility graph} $G=(V,E)$.  Each patient-donor pair, or unpaired altruistic donor, forms a vertex $v \in V$, and a directed edge exists from one vertex to another if the donor at the former can give to the patient at the latter, i.e., are compatible~\cite{Roth04:Kidney,Roth05:Kidney}.

In kidney exchange, patients and donors participate in \emph{cycles} or \emph{chains}.  In a cycle, each participating vertex receives the kidney of the previous vertex.  All transplants in a cycle must be performed simultaneously to ensure participation, and thus are limited to some small length in practice.  This ensures that no donor backs out after her patient has received a kidney but before she has donated her kidney.  Most fielded kidney exchanges---including UNOS---allow only $2$- and $3$-cycles.  In a chain, a donor \emph{without a paired patient} enters the pool, donating his kidney to a patient, whose paired donor donates his kidney to another patient, and so on~\cite{Montgomery06:Domino,Roth06:Utilizing,Rees09:Nonsimultaneous}.  Chains can be executed non-simultaneously\footnote{To see why this is, take the case where a donor backs out of a chain after his paired patient received a kidney, but before his own donation.  Unlike in the case of a broken cycle, no pair in the remaining tail of the planned chain is strictly worse off; that is, no donor was ``used up'' before her paired patient received a kidney.} and thus chains can be longer (but typically not infinite) in length.  Most exchanges---including UNOS---see great gains through the use of such ``altruist-initiated'' chains.

We consider a model that imposes additional structure on an arbitrary compatibility graph.  
For each vertex $v_i \in V$, associate attribute vectors $\vect{d}_i$ and $\vect{p}_i$ with its constituent donor and patient, respectively.  
The $q$th element $d_i^q$ of $\vect{d}_i$ takes on one of a fixed number of types---for example, one of four blood types (O, A, B, AB), or one of a few hundred standard insurance plans.  Then, for $v_i \neq v_j \in V$, we define a \emph{compatibility function} $f(\vect{d}_i,\vect{p}_j)$, a boolean function that returns the compatibility of the donor of $v_i$ with the patient of $v_j$.
We remark that our approach relates to techniques that take advantage of symmetries in a problem; see, e.g., \cite{Meseguer01:Exploiting}.

Given $V$ and associated attribute vectors, we can uniquely determine a compatibility graph $G=(V,E)$ such that $E = \{(v_i, v_j) \suchthat f(\vect{d}_i, \vect{p}_j) = 1 \ \ \forall v_i \neq v_j \in V\}$.  We claim that this model accurately mimics reality, and we later support that claim with strong experimental results on real-world data.  Furthermore, under this new model, certain complexity results central to kidney exchange change (for the better), as we discuss next.

\subsection{The Clearing Problem is Easy (in Theory)}\label{sec:general-computation}

We now tackle the central computational challenge of kidney exchange: the clearing problem.  Well-known to be $\mathcal{NP}$-hard~\cite{Abraham07:Clearing,Biro09:Maximum}, a variety of custom clearing algorithms address adaptations of the clearing problem in practice.\footnote{For an overview of practical approaches to solving the clearing problem, see a recent survey due to Mak-Hau~\shortcite{MakHau15:Kidney}.}  We show that, in our model, the clearing problem itself is solvable in polynomial time.

Formally, we are interested in a polynomial-time algorithm that solves the \emph{\LCYCLECOVER{}} problem---that is, finding the largest disjoint packing of cycles of length at most $L$.  For ease of exposition, in this section we use ``cycles'' to refer to both cycles and chains; indeed, it is easy to see that altruist donors are equivalent to standard patient-donor pairs with a patient who is compatible with all non-altruist vertices in the pool.  Then, a chain is equivalent to a cycle with a ``dummy'' edge returning to the altruist.  Also, again for ease of exposition, we assume the value of a chain of length $L$ is equal to a cycle of length $L$, due to the final donor giving to a patient on the deceased donor waiting list.

Recall that we are working in a model where each vertex $v_i$ belongs to one of a \emph{fixed} number of types determined solely by its attribute vectors $\vect{d}_i$ and $\vect{p}_i$.  Let $\Theta$ be the set of all possible types, and $\theta \in \Theta$ represent one such individual type.  With a slight abuse of notation, we can define a type compatibility function $f(\theta,\theta')=1$ if and only if there is a directed edge between vertices of type $\theta$ and $\theta'$.  (Note that this captures chains and altruist donors as described above.)

A key observation of this section is that any additional edge structure that is imposed on the graph---such as a cycle cover---would be independent of the identity of
\emph{specific} vertices; 
rather, it would only depend on their types, as vertices of the same type have the exact same incoming and outgoing neighborhoods. 
For example, in any cycle cover, if $v_i$ and $v_j$ are two vertices of the same type, 
we can swap $v_i$ and $v_j$
and obtain a feasible cycle cover of the same size. 
This observation drives our theoretical algorithmic results. 

In more detail, every cycle through vertices of $G$ can be interpreted as a closed walk through the type space. 
Every such cycle can be represented by $\typeVec = (\theta_1,\ldots,\theta_\ell) \in \Theta^\ell$, where $\ell$ is the length of the cycle. 
Let us define $f_C$ as the boolean function with 
$f_C({\typeVec}) = 1$ if and only if $$f(\theta_1,\theta_2) = \cdots = f(\theta_{\ell-1}, \theta_{\ell}) = f(\theta_{\ell},\theta_1)=1.$$
Furthermore, for $L \le n = \cardinality{V}$, 
let $\FEASSET(L)$ denote the set of closed walks through the type space of length at most $L$. 
Formally


  {\small\[
    \FEASSET(L) = \bigcup_{\ell = 2}^L \{\typeVec\in \typeSet^\ell \suchthat \CYCLEFEAS(\typeVec)=1\}.
  \]}

Let $\mathcal{C}$ be a cycle cover in $G$, and denote the number of unique \emph{vertices} matched in $\mathcal{C}$ by $\norm{\mathcal{C}}_V$. Suppose $\mathcal{C}$ has cycle cap $L$; then it is equivalent, in our setting, to a vector $\vect{m} \in \Z^{\cardinality{\FEASSET(L)}}_+$,
where, for $\typeVec \in \FEASSET(L)$, 
$m_{\typeVec}$ equals the number of cycles in $\mathcal{C}$
that can be represented in the type space by $\typeVec$.
%
Let $\cardinality{\typeVec}$ be the length of the vector $\typeVec$, and
 $
    \norm{\vect{m}}_V   
      = \sum_{\typeVec \in \FEASSET(L)} m_{\typeVec} \cardinality{\typeVec}
  $.
Then $\norm{\vect{m}}_V= \norm{\mathcal{C}}_V$, that is, $\norm{\vect{m}}_V$ is another way to express the number of matched vertices in the equivalent cycle cover. 

Now consider Algorithm~\ref{alg:main} for \LCYCLECOVER{} in our model,
which we claim is optimal and computationally efficient in our setting. 

\begin{algorithm}[h!]
\begin{enumerate}
  \item 
    $\mathcal{C}^*\gets \emptyset$
  \item 
    \textbf{for} every vector of non-negative integers $\vect{m} \in \Z^{\cardinality{\FEASSET(L)}}_+$ such that 
    $\norm{\vect{m}}_V \le n$
    \begin{itemize}[leftmargin=.15in]
      \item 
        \textbf{if} $\norm{\vect{m}}_V > \norm{\mathcal{C}^*}_V$ and there exists cycle cover $\mathcal{C}$ in $G$ such that
        for each $\typeVec\in\FEASSET(L)$, $\mathcal{C}$ contains $m_{\typeVec}$ cycles 
        of type $\typeVec$,
        \textbf{then} $\mathcal{C}^*\gets \mathcal{C}$
    \end{itemize}
  \item \textbf{return} $\mathcal{C^*}$
\end{enumerate}
\caption{\LCYCLECOVER{}}
\label{alg:main}
\end{algorithm}

\begin{theorem}\label{thm:cycle-cover}
  Given constants $L$ and $|\typeSet|$, Algorithm~\ref{alg:main} is a polynomial-time algorithm for $L$-\textsc{Cycle-Cover}. 
\end{theorem}

\begin{proof}
We start by verifying that Algorithm~\ref{alg:main} is indeed optimal. 
Consider the optimal cycle cover $\mathcal{C}^*$. 
For each $\typeVec\in \FEASSET(L)$, let $m^*_{\typeVec}$ be the number of cycles in $\mathcal{C}^*$ 
that are consistent with the types in $\typeVec$. 
Clearly $\sum_{\typeVec\in \FEASSET(L)}m^*_{\typeVec} \cardinality{\typeVec} \leq n$, as there are only $n$ vertices.
Therefore, Algorithm~\ref{alg:main} considers the collection of numbers $m^*_{\typeVec}$ in Step~2. Because this collection of numbers does induce a valid cycle cover that is of the same size as $\mathcal{C^*}$, the algorithm would update its incumbent cycle cover if it were not already optimal. 

We next analyze the running time of the algorithm. First, note that it is straightforward to check whether the vector 
$\vect{m}$
induces a valid cycle cover. 
Since $\FEASSET(L)$ consists only of valid cycles according to the compatibility function $\CYCLEFEAS{}$, 
we just need to check that there are enough vertices of each type $\type \in \typeSet$ to construct all the cycles that require them. 
This can be checked individually for each $\type \in \typeSet$.
For a particular $\typeVec \in \FEASSET(L)$, the number of vertices of type $\type$ required is $m_{\typeVec}$ 
multiplied by the number of times
type $\type$ appears in $\typeVec$.
Then the sum of these products over all $\typeVec \in \FEASSET(L)$
is at most the number of vertices of type $\type$ in $G$.

Second, there is only a polynomial number of possibilities to construct a collection of numbers 
  $\vect{m} = \{m_{\typeVec}\}_{\typeVec\in\FEASSET(L)}$ 
such that 
  $\norm{\vect{m}}_V \le n$.
Indeed, this number is at most $(n+1)^{\cardinality{\FEASSET(L)}}$.
Moreover, $\cardinality{\FEASSET(L)} \le L\cdot \cardinality{\typeSet}^L$. 
Because $\cardinality{\typeSet}$ and $L$ are constants, $\cardinality{\FEASSET(L)}$ is also a constant. 
The expression $(n+1)^{\cardinality{\FEASSET(L)}}$ is therefore a polynomial in $n$. 
\end{proof}

Even for constant $L$, the running time of Algorithm~\ref{alg:main} is exponential in $\cardinality{\typeSet}$. 
But this is to be expected. Indeed, any graph can trivially be represented using a set $\typeSet$ of types of size $n$, where each vertex has a unique type, and a compatibility function $\CYCLEFEAS{}$ that assigns $1$ to an ordered pair of types if the corresponding edge exists in $G$. Therefore, if the running time of Algorithm~\ref{alg:main} were polynomial in $n$ \emph{and} $\cardinality{\typeSet}$, 
we would solve the general \LCYCLECOVER{} problem in polynomial time---and that problem is $\mathcal{NP}$-hard~\cite{Abraham07:Clearing}.

\subsection{Flipping Attributes is Also Easy (in Theory)}\label{sec:general-computation-cost}

While patients and donors in a kidney exchange are endowed with an initial set of attributes, it may be possible in practice to---at a cost---change some number of those attributes to effect change in the final matching.  For example, the human body naturally tries to reject, to varying degrees, a transplanted organ.  Due to this, nearly all recipients of kidneys are placed on immunosuppressant drugs after transplantation occurs.\footnote{{\tiny\texttt{https://www.kidney.org/atoz/content/immuno}}}  However, \emph{preoperative} immunosuppression can also be performed to increase transplant opportunity---but at some cost to the patient's overall health.

With this in mind, we extend the model of Section~\ref{sec:general-computation} as follows.  Associate with each pair of types $\theta, \theta' \in \Theta$ a cost function $c : \Theta \times \Theta \to \mathbb{R}$ representing the cost of changing a vertex of type $\theta$ to type $\theta'$.  Then, the \LFLIPCOVER{} problem is to find a disjoint packing of cycles of length at most $L$ that maximizes the size of the packing minus the sum of costs spent changing types.  Building on Theorem~\ref{thm:cycle-cover}, this problem is also solvable in polynomial time.

\begin{theorem}\label{thm:flip-and-cover}
Suppose that $L$ and $|\Theta|$ are constants. Then \LFLIPCOVER{} is solvable in polynomial-time.
\end{theorem}
\begin{proof}[Proof sketch]
For any type $\theta \in \Theta$, there are $n_\theta$ vertices. Then, for each of the $(|\Theta|-1)$ choices of which type $\theta' \neq \theta$ to switch to, choose how many vertices from $\theta$ will switch to that type; there are at most $(n_\theta + 1)$ possibilities.  Do this for all types in $\Theta$, resulting in $$\prod_{\theta \in \Theta} (n_\theta + 1)^{|\Theta|-1}\leq (n+1)^{|\Theta|^2}$$ possibilities. Since $|\Theta|$ is a constant, this is polynomial in $n$. 

For each of these polynomially-many type-switch possibilities, we can compute the optimal cycle cover in polynomial time using Algorithm~\ref{alg:main}, and subtract $c(\theta,\theta')$ for each vertex that switched from $\theta$ to $\theta'$. Taking the best of these solutions gives the optimal solution in polynomial time. 
\end{proof}

\section{A Concrete Instantiation: Thresholding}
\label{sec:theory}
As motivated in Sections~\ref{sec:intro} and~\ref{sec:model}, compatibility in real kidney exchange graphs is determined by patient and donor attributes, such as blood or tissue type.  In particular, if an attribute for a donor and patient is in conflict, they are deemed incompatible. Motivated by that reality, in this section, we associate with each patient and donor a bit vector of length $k$, and count incompatibilities based on any shared activated bits between a patient and potential donor.

As a concrete example, consider human blood types. 
At a high level, human blood contains A antigens, B antigens, both (type AB), or neither (type O).  
AB-type patients can receive from any donor, 
A-type (B-type) can receive from O-type and A-type (B-type) donors, 
and O-type patients can only receive from O-type donors.  
In our bit model, this is represented with $k=2$.
The first bit represents compatibility with A antigens and the second bit represents compatibility with B antigens.
Thus, the type space 
  $\Theta = 2^{\{\texttt{has-A},\texttt{has-B}\}} \times 2^{\{\texttt{no-A},\texttt{no-B}\}}$; in general, $|\Theta| = 2^{2k}$.

Formally, unless otherwise stated, throughout this section $G$ will refer to a directed graph
with vertex set $V = [n] := \{1,\ldots,n\}$ and edge set $E$, 
and with each $i \in V$ associated with two $k$-bit vectors $\vect{d}_i, \vect{p}_i \in \{0,1\}^k$.
Let 
  $\ConflictSet_d(i) = \{q \in [k] \suchthat \vect{d}_{iq} = 1\}$
be the set of conflict bits for the donor associated with vertex $i \in V$, and similarly let 
  $\ConflictSet_p(i) = \{q \in [k] \suchthat \vect{p}_{iq} = 1\}$.
For $i,j \in V$ such that $i \ne j$, the threshold feasibility function $\fthresh{t}$ is defined as
  \begin{align*}
    \fthresh{t}(\vect{d}_i,\vect{p}_j) =
      \begin{cases}
        1 & \mbox{if $\abs{\ConflictSet_d(i) \cap \ConflictSet_p(j)} \le t$,} \\
        0 & \mbox{otherwise.}
      \end{cases}.
  \end{align*}
Note that $\abs{\ConflictSet_d(i) \cap \ConflictSet_p(j)} \le t$ if and only if $\langle \vect{d}_i, \vect{p}_j \rangle \le t$.

Kidney exchange graphs constructed using threshold compatibility functions 
are closely related to complements of \emph{intersection graphs}~\cite{McKMcM99}, 
which are graphs that have a set associated with each vertex and an edge between two vertices if and only if the sets intersect.
Given $t \in \mathbb{N}$,
the function $\fthresh{t}$ is related to \emph{$p$-intersection graphs}~\cite{ChuWes94,EatGouRod96},
where an edge connects two vertices if their corresponding sets intersect in at least $p \ge 1$ elements.
%

In particular, our model is similar to that of \emph{intersection digraphs}~\cite{SenDasRoyWes89}, 
or equivalently \emph{bipartite intersection graphs}~\cite{HarKabMcM82}, 
both also considered in~\cite{Orlin77}.
These have mainly been studied under the assumption that the sets used to represent the graph have the ``consecutive ones'' property,
i.e., each set is an interval from the set of integers.
Our model is more general: we do not place such an assumption on the set of conflict bits.
Moreover, most treatments of intersection digraphs consider loops on the vertices,
whereas in our thresholding model, whether or not donor $i$ and patient $i$ are compatible is not considered.
In addition, the directed and bipartite intersection graph literature has focused on the case that $t = 0$ (in our terminology).
To the best of our knowledge, this paper is the first treatment \emph{$p$-intersection digraphs}, and certainly their first real-world application.


\subsection{Existence of Small Representations}
It is natural to ask for what values of $t$ and $k$ we can select vertices with 
bit vectors $\vect{d}_i$ and $\vect{p}_i$ of length $k$ such that $\THRESHOLDt{t}$ can create \emph{any} graph of a specific size? 

Formally, we say that 
  $G$
  is \emph{$(k,t)$-representable} (by feasibility function $\fthresh{t}$)
  if, for all $i \in V$ there exist $\vect{d}_i, \vect{p}_i \in \{0,1\}^k$ such that for all $j_1\in V$, $j_2\in V\setminus\{j_1\}$, $(j_1,j_2)\in E$ if and only if $\fthresh{t}(\vect{d}_{j_1},\vect{p}_{j_2}) = 1$. 

It is known~\cite{ErdGooPos66} that any graph can be represented as an intersection graph with $k \le n^2 / 4$.
Yet, we show next that, in our model, $k \le n$ suffices to represent any graph. It is akin to a result on the \emph{term rank} of the adjacency matrix of $\overline G$~\cite[Thm~6.6]{Orlin77}.

\begin{theorem}\label{thm:k-upper-bound}
  Let $G = (V,E)$ be a digraph on $n$ vertices.
  Let $n_1$ be the number of vertices with outgoing edges, 
  Let $n_2$ be the number of vertices with incoming edges, and $n' = \min\{n_1+1,n_2+1,n\}$.
  Then $G$ can be $(n',0)$-represented.
\end{theorem}

\begin{proof}
%
  We first show that the graph can be $(n_1+1,0)$-represented.
  Assume without loss of generality that vertices $1,\ldots,n_1$ have outgoing edges.
  We show how to set $\vect{d}_i, \vect{p}_i \in \{0,1\}^{n_1+1}$ for each vertex $i$ in $V$.
  To set the donor attributes, for each $i \in [n_1]$, let $\vect{d}_i$ be $e_i$, the $i$th standard basis vector, 
  i.e., the vector of length $n_1+1$ with a $1$ in the $i$th coordinate and $0$ everywhere else.
  For $i > n_1$, set $\vect{d}_i$ to be $e_{n_1+1}$.
  For the patient attributes of vertex $j \in [n]$,
  for each $i \in [n]$ such that $(i,j) \in E$, set $\vect{p}_{ji} = 0$, and set $\vect{p}_{ji} = 1$ otherwise.
  Note that if all the vertices have outgoing edges, then $n_1 = n$ unit vectors suffice.
  A similar approach works to $(\min\{n,n_2+1\},0)$-represent $G$, by using the $n_2$ unit vectors 
  as the patient vectors of those vertices with incoming edges, and (if needed) one additional unit vector for any remaining vertices.
  In both of these cases, $\langle \vect{d}_i, \vect{p}_j \rangle = 0$ if and only if $(i,j) \in E$, which represents $G$ by $\THRESHOLDt{0}$.
\end{proof}

In particular, Theorem~\ref{thm:k-upper-bound} implies that any graph is $(n,0)$-representable. The next theorem shows a matching lower bound.
The same construction and bound also hold if loops are considered~\cite{SenDasRoyWes89}.

\begin{theorem}
\label{thm:not0impl}
  For any $n \ge 3$, there exists a graph on $n$ vertices that is not $(k,0)$-representable for all $k < n$.
\end{theorem}

\begin{proof}
  Define $G$ to be the digraph on $n$ vertices, $V = [n]$,
  with an edge from vertex $i$, for each $i \in V$, to every vertex except $i-1$ (and itself),
  where vertex $n$ is also referred to as vertex $0$.

  Assume that $G$ is $(k,0)$-representable, and 
  consider vertex $1$. 
  Since $(1,n) \notin E$, and $(i,n) \in E$ for all $i \notin \{1,n\}$, 
  there exists a conflict bit $q_1 \in \ConflictSet_d(1) \cap \ConflictSet_p(n)$ such that $q_1 \notin \ConflictSet_p(V \setminus \{1,n\})$.
  More generally, there exists such a conflict bit $q_i$ for all $i \in V$.

  We claim that these conflict bits are all unique, which directly implies that $k \ge n$. 
  Indeed, otherwise we can assume that $q_1 = q_i$ for some $i \ne 1$ 
  (without loss of generality, as the graph is symmetric subject to cyclic permutations).
  But then $(1,i-1)$ and $(i,n)$ do not appear as edges in $G$, which is not true for any $i \in V \setminus \{1\}$.
\end{proof}

More generally, it is easy to see that any graph that is $(k,0)$-representable is also $(k+t,t)$-representable for any $t \ge 0$. 
Indeed, simply take the $(k,0)$-representation of the graph, and append $t$ ones to every vector. 
Together with Theorem~\ref{thm:k-upper-bound}, this shows that any graph is $(n+t,t)$-representable. 
However, the lower bound given by Theorem~\ref{thm:not0impl} does not extend to $t>0$. 
We conjecture that for any $n$ and $t$ there exists a graph that can only be $(k,t)$-represented with 
$k=\Omega(n)$---this remains an open question.

\subsection{Computational Issues}

Given a real compatibility graph with $n$ vertices, 
we know by Theorem~\ref{thm:k-upper-bound} that we can $(k,0)$-represent that graph for $k=n$.  
But, in practice, how large of a $k$ is actually needed?

Various problems related to intersection graphs are $\mathcal{NP}$-complete
for general graphs~\cite{KouStoWon78,Orlin77},
but we work in a setting with additional structure.
And while we do not show that finding a $(k,t)$-representation is $\mathcal{NP}$-hard,
we do show that a slightly harder problem, which we refer to as 
{\sc $(k,t)$-Representation-with-Ignored-Edges}, is $\mathcal{NP}$-hard. 
Given an input of a directed graph $G = (V,E)$, a subset $F$ of $\binom{V}{2}$, and integers $k \ge 1$ and $t \ge 0$,
this problem asks whether there exist bit vectors $\vect{d}_i$ and $\vect{p}_i$ of length $k$ for each $i \in V$
such that for any $(i,j) \in F$, we have $(i,j) \in E$ if and only if $\langle \vect{d}_i, \vect{p}_j \rangle \le t$.

\begin{theorem}\label{thm:np-hardness-main-paper}
  The {\sc $(k,t)$-Representation-with-Ignored-Edges} problem is $\mathcal{NP}$-complete.
\end{theorem}

The theorem's nontrivial proof is relegated to Appendix~\ref{app:proofs}.
Here we give a proof sketch.
One major idea is the construction of a \emph{bit-grounding gadget} $G_k$---a subgraph where the bits are set uniquely (up to permutations) in any valid representation, and can be used to set the bits in other vertices. The gadget has $\binom{k}{2}$ vertices; we prove that there is a unique (up to permutations) $(k,1)$-representation of $G_k$, where each donor vector has a unique pair of ones, and similarly for each patient vector. 
Figure~\ref{fig:gadget} shows $G_4$. 

Then, we prove $\mathcal{NP}$-hardness by reduction from 3SAT. In the constructed instance of our problem, we set the threshold to $1$. The crux of the reduction is to add a vertex for each clause in the given 3SAT formula, where in the patient vector, the bit corresponding to each literal in the clause is set to $1$. This can be done by connecting the vertex to the bit-grounding gadget. Moreover, there is a special vertex that does not have outgoing edges to any of the clause vertices. This means that it must have a $1$ in a position that corresponds to one of the literals in each clause. A different part of the construction ensures that there is at most a single $1$ in the two positions corresponding to a variable and its negation. Therefore, a valid assignment of the donor bits corresponds to a satisfying assignment for the 3SAT formula. 

\begin{figure}[htp]
  \begin{center}
  \begin{tikzpicture}[scale=1.0]
  \tikzstyle{ns} = [draw,shape=circle,outer sep=0,inner sep=1,minimum size=12,font=\small]
  \tikzstyle{nsquare} = [draw,shape=rectangle,outer sep=0,inner sep=1,minimum size=12,font=\small]
  
  \node [ns] (123) at (0,-1.5) {$1$};
  \node [ns] (145) at (1,0.8) {$2$};
  \node [ns] (246) at (-0.5,-0.5) {$3$};
  \node [ns] (356) at (0,1.5) {$4$};
  \node [nsquare,label={[label distance=0cm,align=left]90:{$d_1:1100$\\$p_1:1010$}}] (1) at (-1,2) {$1$};
  \node [nsquare,label={[label distance=0cm,align=left]90:{$d_2:1010$\\$p_2:1001$}}] (2) at (1,2) {$2$};
  \node [nsquare,label={[label distance=0cm,align=left]0:{$d_3:1001$\\$p_3:0110$}}] (3) at (2,0) {$3$};
  \node [nsquare,label={[label distance=0cm,align=left]270:{$d_4:0110$\\$p_4:0101$}}] (4) at (1,-2) {$4$};
  \node [nsquare,label={[label distance=0cm,align=left]270:{$d_5:0101$\\$p_5:0011$}}] (5) at (-1,-2) {$5$};
  \node [nsquare,label={[label distance=0cm,align=left]180:{$d_6:0011$\\$p_6:1100$}}] (6) at (-2,0) {$6$};
  \draw  (1) edge[->,dotted] node {} (6);
  \draw  (2) edge[->,dotted] node {} (1);
  \draw  (3) edge[->,dotted] node {} (2);
  \draw  (4) edge[->,dotted] node {} (3);
  \draw  (5) edge[->,dotted] node {} (4);
  \draw  (6) edge[->,dotted] node {} (5);
  
  \draw  (123) edge[<-,dotted] node {} (4);
  \draw  (123) edge[<-,dotted] node {} (5);
  \draw  (123) edge[<-,dotted] node {} (6);
  
  \draw  (145) edge[<-,dotted] node {} (2);
  \draw  (145) edge[<-,dotted] node {} (3);
  \draw  (145) edge[<-,dotted] node {} (6);
  
  \draw  (246) edge[<-,dotted] node {} (1);
  \draw  (246) edge[<-,dotted] node {} (3);
  \draw  (246) edge[<-,dotted] node {} (5);
  
  \draw  (356) edge[<-,dotted] node {} (1);
  \draw  (356) edge[<-,dotted] node {} (2);
  \draw  (356) edge[<-,dotted] node {} (4);
  \end{tikzpicture}
  \end{center}
  \caption{Gadget $G_4$ with a subset of \emph{non}-edges shown; all edges between circle vertices are also not in $E$.}
  \label{fig:gadget}
\end{figure}
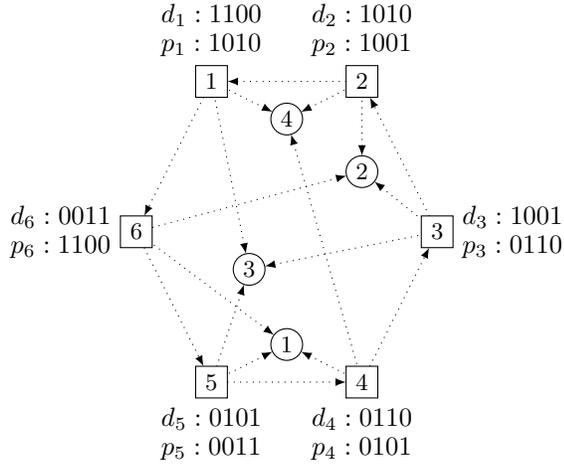

\section{Computing Small Representations of Real Kidney Exchange Compatibility Graphs}\label{sec:exp}

In this section, we test our hypothesis that real compatibility graphs can be represented by a substantially smaller number of attributes than required by the worst-case theoretical results of Section~\ref{sec:theory}.  We begin by presenting general math programming techniques to determine, given $k,t\in\mathbb{Z}$, whether a specific graph $G=(V,E)$ is $(k,t)$-representable.  We then show on real and generated compatibility graphs from the UNOS US-wide kidney exchange that small $k$ suffices for $(k,0)$-representation, and conclude by exploring the allowance of greater thresholds $t$ on match size.  Even small thresholds $t>0$ result in substantial societal gain.\footnote{%
All code for this section can be found at \url{https://github.com/JohnDickerson/KidneyExchange}.%
}

\subsection{Mathematical Programming Formulations}
Implementation of $\fthresh{t}$ can be written succinctly as a quadratically-constrained discrete feasibility program (QCP) with $2k|V|$ binary variables, given as~\ref{eq:qcp} below.

{\small%
\begin{equation}\label{eq:qcp}
\begin{array}{rrr}
  & \langle \vect{d}_i, \vect{p}_j \rangle \leq t & \forall (v_i, v_j) \in E \\
  & \langle \vect{d}_i, \vect{p}_j \rangle \ge (t+1) & \forall (v_i, v_j) \not\in E \\
  & \vect{d}_i, \vect{p}_i \in \{0,1\}^k & \forall v_i \in V
\end{array}\tag{M1}
\end{equation}
}

The constraint matrix for this program is not positive semi-definite, and thus the problem is not convex.  Exploratory use of heuristic search via state-of-the-art integer nonlinear solvers~\cite{Bonami08:Algorithmic} resulted in poor performance (in terms of runtime and solution quality) on even small graphs.  With that in mind, and motivated by the presence of substantially more mature integer \emph{linear} program (ILP) solvers, we linearize~\ref{eq:qcp}, presented as~\ref{eq:ilp} below.

{\small%
\begin{equation}\label{eq:ilp}
\!\!\begin{array}{rrr}
  \min        & \sum_{v_i \in V} \sum_{v_j \neq v_i \in V} \xi_{ij}  & \\
  \text{s.t.} & d^q_i \geq c^q_{ij} \land p^q_j \geq c^q_{ij} & \forall v_i \neq v_j \in V, q \in [k] \\
  & d^q_i + p^q_j \leq 1 + c^q_{ij} & \forall v_i \neq v_j \in V, q \in [k] \\
  & \sum_{q} c^q_{ij} \leq t + (k-t)\xi_{ij} & \forall (v_i, v_j) \in E \\
  & \sum_{q} c^q_{ij} \geq (t+1)\xi_{ij} & \forall (v_i, v_j) \in E \\
  & \sum_{q} c^q_{ij} \geq t+1 - k\xi_{ij} & \forall (v_i, v_j) \not\in E \\
  & \sum_{q} c^q_{ij} \leq k - (k-t)\xi_{ij} & \forall (v_i, v_j) \not\in E \\
  & d^{q}_i, p^{q}_i \in \{0,1\} & \forall v_i \in V, q \in [k] \\
  & c^{q}_{ij}, \xi_{ij} \in \{0,1\} & \forall v_i \neq v_j \in V, q \in [k]
\end{array}\tag{M2}
\end{equation}
}

\ref{eq:ilp} generalizes~\ref{eq:qcp}; while~\ref{eq:qcp} searches for a feasible solution to the $(k,t)$-representation problem,~\ref{eq:ilp} finds the ``best'' (possibly partially-incorrect) solution by minimizing the total number of edges that exist in the solution but not in the base graph $G$, or do not exist in the solution but do in $G$.  This flexibility may be desirable in practice to strike a tradeoff between small $k$ and accuracy of representation.

Interestingly, neither the fully general ILP nor its (smaller) instantiations for the special cases of feasibility and/or threshold $t=0$ were solvable by a leading commercial ILP solver~\cite{CPLEX12.6} within $12$ hours for even small graphs, primarily due to the model's loose LP relaxation.  Indeed, the model we are solving is inherently logical, which is known to cause such problems in traditional mathematical programming~\cite{Hooker02:Logic}.  With that in mind, we note that the special case of $t=0$ can be represented compactly as a satisfiability (SAT) problem in conjunctive normal form, given below as~\ref{eq:cnf-sat}.

{\small%
\begin{equation}\label{eq:cnf-sat}
    \begin{array}{rr}
      \bigwedge\limits_{q \in [k]} (\neg d_i^q \lor \neg p_j^q)   & \forall (v_i, v_j) \in E \\
      \begin{array}{c}
        (z^1_{ij} \lor z^2_{ij} \lor \ldots \lor z^k_{ij}) \ \land \\ \bigwedge\limits_{q \in [k]}\left[ (\neg z^q_{ij} \lor d_i^q) \land (\neg z^q_{ij} \lor p_j^q) \right]
      \end{array} & \forall (v_i, v_j) \not\in E \\
  \end{array}\tag{M3}
\end{equation}
}

This formulation maintains two sets of clauses: the first set enforces no bit-wise conflicts for edges in the underlying graph, while the second set enforces at least one conflict via $k$ auxiliary variables $z^{\cdot}_{ij}$ for each possible edge $(v_i,v_j)\not\in E$.  \ref{eq:cnf-sat} was amenable to parallel SAT solving~\cite{Biere14:Yet}.  Next, we present results on real graphs with this formulation.

\subsection{$(k,0)$-representations of Real Graphs}
Can real kidney exchange graphs be represented by a small number of attributes?  To answer that question, we begin by testing on real match run data from the first two years of the United Network for Organ Sharing (UNOS) kidney exchange, which now contains \num{143} transplant centers, that is, 60\% of all transplant centers in the US.  We translate each compatibility graph into a CNF-SAT formulation according to~\ref{eq:cnf-sat}, and feed that into a SAT solver~\cite{Biere14:Yet} with access to \num{16}GB of RAM, \num{4} cores, and \num{60} minutes of wall time.  (Timeouts are counted---conservatively against our paper's qualitative message---as negative answers.)  

Figure~\ref{fig:phase-transition-frac-k} shows a classical phase transition from unsatisfiability to satisfiability as $k$ increases as a fraction of graph size, as well as an associated substantial increase in computational intractability centered around that phase transition.  This phenomenon is common to many central problems in artificial intelligence~\cite{Cheeseman91:Where,Hogg96:Phase,Walsh11:Where}.  Indeed, we see that substantially fewer than $|V|$ attributes are required to represent real graphs; compare with the lower bound of Theorem~\ref{thm:not0impl}.

\begin{figure}[ht!bp]
  \centering
  \includegraphics[width=\linewidth]{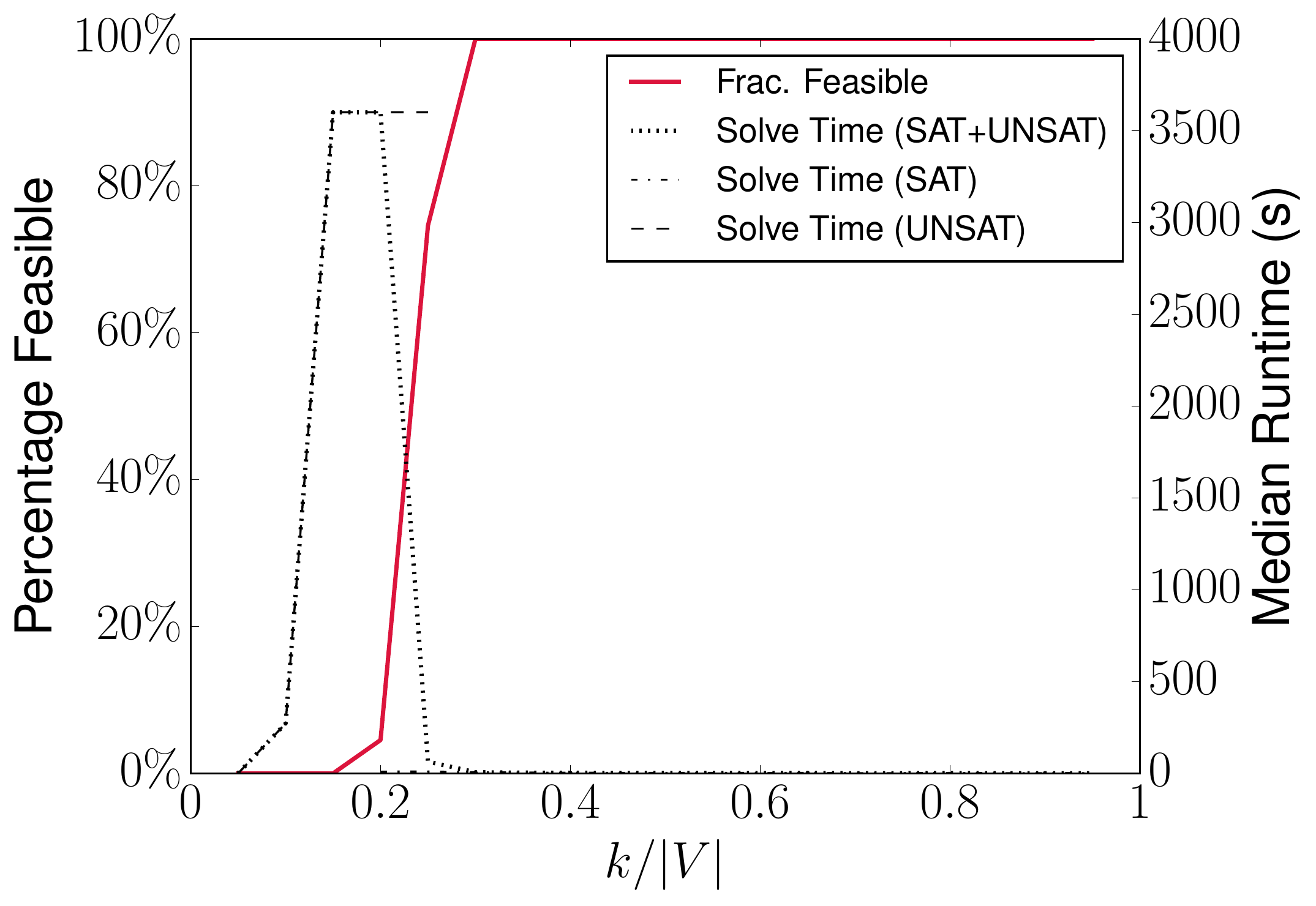}
  \caption{Classical hardness spike near the phase transition for $(k,0)$-representation on real UNOS graphs.}
  \label{fig:phase-transition-frac-k}
\end{figure}

Figure~\ref{fig:sat-by-size} explores the minimum $k$ required to represent each graph as a function of $|V|$, compared against the theoretical upper bound of $|V|$.  The shaded area represents those values of $k$ where the SAT solver timed out; thus, the reported values of $k$ are a conservative \emph{upper} bound on the required minimum, which is still substantially lower than $|V|$.

\begin{figure}[ht!bp]
  \centering
  \includegraphics[width=\linewidth]{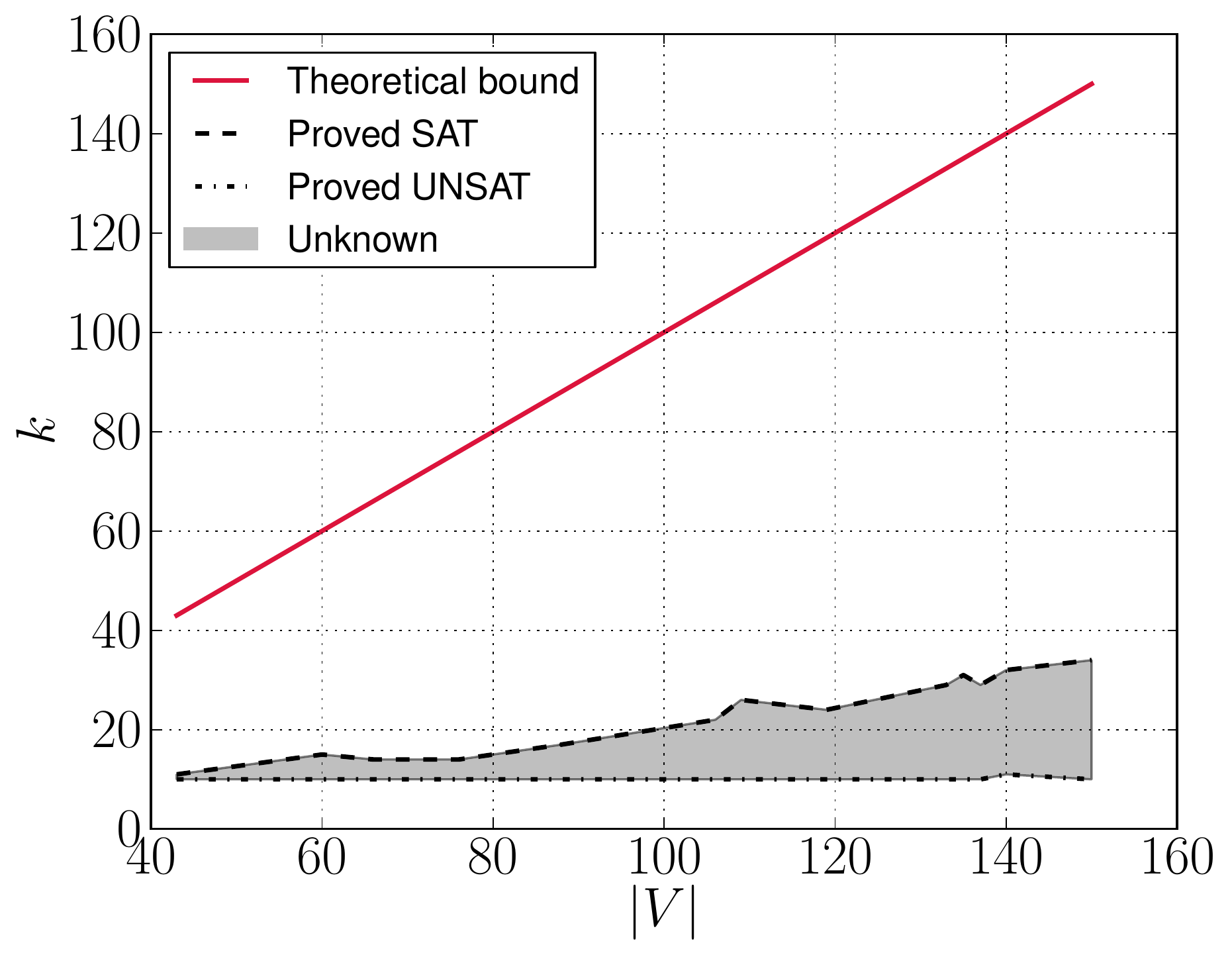}
  \caption{Comparison of number of bits (y-axis) required to $(k,0)$-represent real UNOS compatibility graphs of varying sizes (x-axis).  The theoretical bound of $k=|V|$ is shown in red; it is substantially higher than the conservative upper bound of $k$ solved by our SAT solver (upper dotted line).}
  \label{fig:sat-by-size}
\end{figure}

\subsection{Thresholding Effects on Matching Size}
A motivation of this work is to provide a principled basis for optimally ``flipping bits'' of participants (via, e.g., immunosuppresion) in fielded kidney exchanges, in the hope that additional edges in the compatibility graph will result in gains in the final algorithmic matchings.  We now explore this line of reasoning---that is, increasing the $t$ in $\fthresh{t}$ instead of the $k$, which is now endogenous to the underlying model---on realistic generated UNOS graphs of varying sizes.

Figure~\ref{fig:pct-matched} shows the effect on the percentage of patient-donor pairs matched by $2$- and $3$-cycles as a global threshold $t$ is raised incrementally from $t=0$ (the current status quo) to $t=5$.  Intuitively, larger compatibility graphs result in a higher fraction of pairs being matched; however, a complementary approach---making the graph denser via even small increases in $t$---also results in tremendous efficiency gains of $3$--$4$x (depending on $|V|$) over the baseline for $t=0$, and quickly increasing to all pairs being matched by $t=5$.

\begin{figure}[ht!bp]
  \centering
  \includegraphics[width=\linewidth]{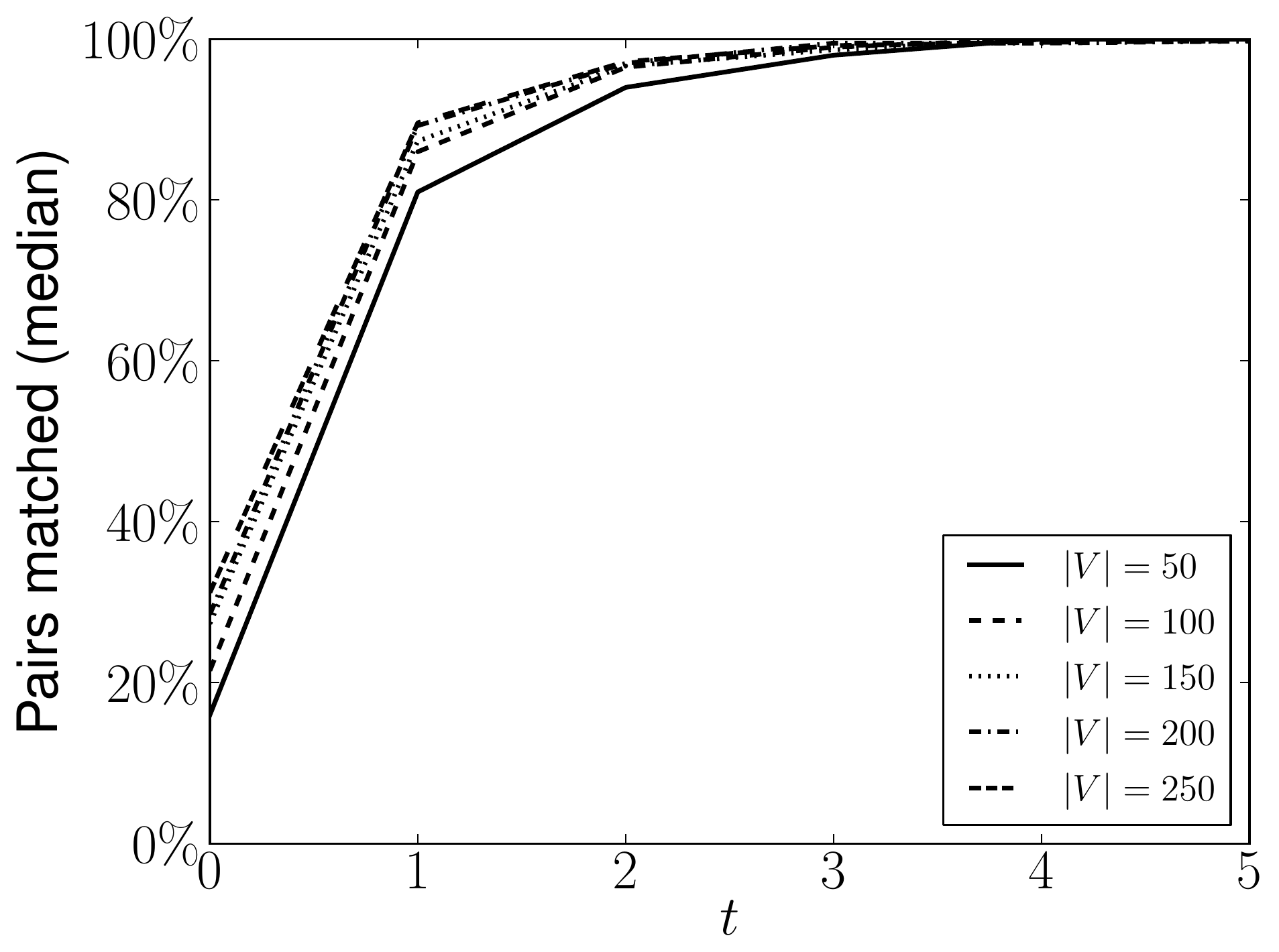}
  \caption{Pairs matched (\%, y-axis) in generated UNOS graphs of varying sizes (lines), as $t$ increases (x-axis).}
  \label{fig:pct-matched}
\end{figure}

We note that any optimal matching found after increasing a global threshold $t$ could also be created by paying to change at most $t$ bits per vertex in a graph; however, the practical \emph{selection} of the minimum-sized set of at most $t$ bits per vertex such that the size of the resulting optimal matching is equal to that found under the global threshold of $t$ is a difficult two-stage problem
and is left as future research.  The large efficiency gains realized by moving from $\fthresh{0}$ to even $\fthresh{1}$ motivate this direction of research.

\section{Conclusions \& Future Research}

Motivated by the increasing size of real-world kidney exchanges, we presented a compact approach to modeling kidney exchange compatibility graphs.  Our approach is intimately connected to classical intersection graph theory, and can be viewed as the first exploration and practical application of $p$-intersection digraphs.  We gave necessary and sufficient conditions for losslessly shrinking the representation of an arbitrary compatibility graph in this model. Real compatibility graphs, however, are not arbitrary, and are created from characteristics of the patients and donors; using real data from the UNOS US-wide kidney exchange, we showed that using only a small number of attributes suffices to represent real graphs.  This observation is of potential practical importance; if real graphs can be represented by a constant number of attributes, then central $\mathcal{NP}$-hard problems in general kidney exchange are solvable in polynomial time.

This paper only addresses the representation of static compatibility graphs; in reality, exchanges are dynamic, with patients and donors arriving and departing over time~\cite{Unver10:Dynamic}.  Extending the proposed method to cover time-evolving graphs is of independent theoretical interest, but may also be useful in speeding up the (presently-intractable) dynamic clearing problem~\cite{Awasthi09:Online,Dickerson12:Dynamic,Anderson14:Stochastic,Dickerson15:FutureMatch,Glorie15:Robust}.  Better exact and approximate methods for computing ($k,t$)-representations of graphs would likely be a prerequisite for that research.
Adaptation of the theoretical results to models of lung, liver, and multi-organ exchange would also be of practical use~\cite{Ergin14:Lung,Ergin15:Liver,Luo15:Mechanism,Dickerson16:Multi}.

\section*{Acknowledgments}
This material is based on work supported by the National Science Foundation (NSF) under grants CCF-1215883, CCF-1525932, IIS-1350598, IIS-1617590, IIS-1320620, and IIS-1546752, the Army Research Office (ARO) under award W911NF-16-1-0061, the Office of Naval Research (ONR) under award N00014-16-1-3075, a Sloan Research Fellowship, a Facebook Fellowship, and a Siebel Scholarship.  It made use of XSEDE computing resources provided by the Pittsburgh Supercomputing Center.  The authors thank participants at EXPLORE-16, EURO-16, and INFORMS-16 sessions, and particularly the anonymous reviewers at AAAI-17 and James Trimble, for helpful discussion. 




\bibliography{dairefs,intgraphs}
\bibliographystyle{aaai}

\appendix
\section{Additional Proofs}\label{app:proofs}
In this section, we provide the full proof of Theorem~\ref{thm:np-hardness-main-paper}.
Recall the {\sc $(k,t)$-representation with Ignored Edges}: given an input of a directed graph $G = (V,E)$, a subset $F$ of $E$, and integers $k \ge 1$ and $t \ge 0$,
this problem asks whether there exist bit vectors $\vect{d}_i$ and $\vect{p}_i$ of length $k$ for each $i \in V$
such that the $\{i,j\} \in F$ if and only if $\langle \vect{d}_i, \vect{p}_j \rangle \le t$.

Consider the gadget $G_k$ defined as follows on a graph on $\binom{k}{2} + k$ vertices.
Let $G_k^1$ be the graph defined in Theorem~\ref{thm:not0impl} on $\binom{k}{2}$ vertices,
i.e., the complement of a directed cycle on this many vertices.
Associate with each vertex $u \in G_k^1$ a unique element from $\binom{[k]}{2}$
(all subsets of $[k]$ of size 2).
Let $G_k^2$ be an independent set of $k$ vertices.
For each vertex $i \in G_k^2$, $i \in [k]$, add an incoming edge into $i$ from $u \in G_k^2$
if and only if $i \in S_u$.
Figure~\ref{fig:gadget} shows $G_4$.

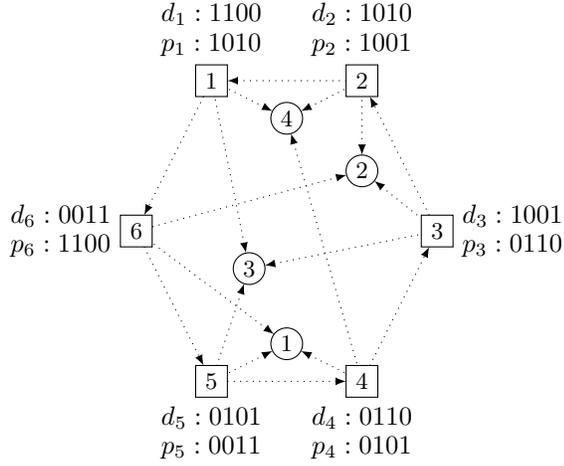
\begin{figure}[htp]
  \begin{center}
  \begin{tikzpicture}
  \tikzstyle{ns} = [draw,shape=circle,outer sep=0,inner sep=1,minimum size=12,font=\small]
  \tikzstyle{nsquare} = [draw,shape=rectangle,outer sep=0,inner sep=1,minimum size=12,font=\small]
  
  \node [ns] (123) at (0,-1.5) {$1$};
  \node [ns] (145) at (1,0.8) {$2$};
  \node [ns] (246) at (-0.5,-0.5) {$3$};
  \node [ns] (356) at (0,1.5) {$4$};
  \node [nsquare,label={[label distance=0cm,align=left]90:{$d_1:1100$\\$p_1:1010$}}] (1) at (-1,2) {$1$};
  \node [nsquare,label={[label distance=0cm,align=left]90:{$d_2:1010$\\$p_2:1001$}}] (2) at (1,2) {$2$};
  \node [nsquare,label={[label distance=0cm,align=left]0:{$d_3:1001$\\$p_3:0110$}}] (3) at (2,0) {$3$};
  \node [nsquare,label={[label distance=0cm,align=left]270:{$d_4:0110$\\$p_4:0101$}}] (4) at (1,-2) {$4$};
  \node [nsquare,label={[label distance=0cm,align=left]270:{$d_5:0101$\\$p_5:0011$}}] (5) at (-1,-2) {$5$};
  \node [nsquare,label={[label distance=0cm,align=left]180:{$d_6:0011$\\$p_6:1100$}}] (6) at (-2,0) {$6$};
  \draw  (1) edge[->,dotted] node {} (6);
  \draw  (2) edge[->,dotted] node {} (1);
  \draw  (3) edge[->,dotted] node {} (2);
  \draw  (4) edge[->,dotted] node {} (3);
  \draw  (5) edge[->,dotted] node {} (4);
  \draw  (6) edge[->,dotted] node {} (5);
  
  \draw  (123) edge[<-,dotted] node {} (4);
  \draw  (123) edge[<-,dotted] node {} (5);
  \draw  (123) edge[<-,dotted] node {} (6);
  
  \draw  (145) edge[<-,dotted] node {} (2);
  \draw  (145) edge[<-,dotted] node {} (3);
  \draw  (145) edge[<-,dotted] node {} (6);
  
  \draw  (246) edge[<-,dotted] node {} (1);
  \draw  (246) edge[<-,dotted] node {} (3);
  \draw  (246) edge[<-,dotted] node {} (5);
  
  \draw  (356) edge[<-,dotted] node {} (1);
  \draw  (356) edge[<-,dotted] node {} (2);
  \draw  (356) edge[<-,dotted] node {} (4);
  \end{tikzpicture}
  \end{center}
  \caption{Gadget $G_4$ with a subset of \emph{non}-edges shown; all edges between circle vertices (those in $G_4^2$) are also not in $E$.}
  \label{fig:gadget}
\end{figure}

Denote the \emph{donor neighborhood} of $i \in V$
by $N_d(i) = \{j \in V \suchthat (i,j) \in E, i \ne j\}$,
i.e., the set of patients compatible with  donor $i$.
Similarly, the \emph{patient neighborhood} of $j \in V$
is $N_p(j) = \{i \in V \suchthat  \{i,j\} \in E, i \ne j\}$.

\begin{lemma}
  There is a unique (up to permutations) $(k,1)$-representation of $G_k$.
\end{lemma}
\begin{proof}  First consider $G_k^1$.
  For all $u \in V(G_k^1)$,
  since $\{u,u-1\} \notin E(G_k^1)$, and the compatibility function is $\fthresh{1}$,
  there exist two distinct conflict bits $q^u_1$ and $q^u_2$ in $\ConflictSet_d(u) \cap \ConflictSet_p(u-1)$.
  Moreover, for any $u,v$ distinct, $\{q^u_1, q^u_2\} \ne \{q^v_1, q^v_2\}$.
  Otherwise, $\{q^{u}_1, q^{u}_2\} \subseteq \ConflictSet_p(v-1)$ and $\{q^{v}_1, q^{v}_2\} \subseteq \ConflictSet_p(u-1)$,
  but at least one of the edges $\{u,v-1\}$ or $\{v,u-1\}$ exists in $G_k^1$.
  
  In addition, $\abs{\ConflictSet_d(u)} = 2$ for all $u \in V(G_k^1)$.
  Suppose not, and there exists a third distinct (from $q^{u}_1$ and $q^{u}_2$) conflict bit $q^{u}_3$ in $\ConflictSet_d(u)$.
  As the number of vertices is $\binom{k}{2}$, 
  there exists a vertex $v_1$ with $\{q^{v_1}_1,q^{v_1}_2\} = \{q^{u}_1, q^{u}_3\}$,
  and a (different) vertex $v_2$ with $\{q^{v_2}_1,q^{v_2}_2\} = \{q^{u}_2, q^{u}_3\}$.
  Then $\{u,v_1-1\}$ and $\{u,v_2-1\}$ are both not in $E(G_k^1)$.
  However, $u$ has edges to all vertices except itself and $u-1$,
  which is a contradiction, as $u$, $v_1$, and $v_2$ are all distinct.
  From this, it also follows that $\abs{\ConflictSet_p(u)} = 2$.
  
%
  
  We have thus shown that every vertex $u \in G_k^1$ has exactly two bits set to one in its donor attribute vector,
  with a unique pair of bits per vertex, and $\ConflictSet_d(u) = \ConflictSet_p(u-1)$.
  However, without more structure, it is not possible to tell in which donor vectors a particular conflict bit appears.
  The additional graph $G_k^2$ allows us to identify this, up to permutations.
  
  Since there are no outgoing edges from any of the vertices in $G_k^2$, 
  and every pair of bits in $\binom{[k]}{2}$ appears in exactly one patient vector of a vertex in $G_k^1$, 
  each donor vector in $G_k^2$ must be the all-ones vector of length $k$.
  
  Consider vertex $i \in [k]$ in $G_k^2$.
  It has an incoming edge from each vertex $u \in V(G_k^1)$ such that $i \in S_u$
  and it is missing the $\binom{k-1}{2}$ other possible incoming edges from $G_k^1$
  (note that the labeling of the vertices, as well as the choices of the sets $S_u$, 
  are made without any knowledge of the bit-vectors associated with the vertices).
  We next show that $\abs{\cap_{u \in N_p(i)} \ConflictSet_d(u)} = 1$.
  That this quantity is at most 1 is clear, as $\ConflictSet_d(u)$ and $\ConflictSet_d(v)$ 
  intersect in at most one conflict bit for all $u, v \in V(G_k^1)$, $u \ne v$.
  If this quantity were 0, then for some $u,v \in N_p(i)$, $\ConflictSet_d(u) \cap \ConflictSet_d(v) = \emptyset$.
  But then at least two zeroes would appear in $\ConflictSet_p(i)$,
  which is a contradiction as it implies that $i$ would have more than $k$ incoming edges.
  Thus, the patient vector $\vect{p}_i$ for $i \in V(G_k^2)$ has exactly one zero and ones elsewhere.
  Moreover, since $N_p(i) \ne N_p(j)$ for any distinct $i,j \in [k]$, it follows that $\vect{p}_i \ne \vect{p}_j$,
  so each patient vector is distinct and the position of its only zero is unique.
\end{proof}

\begin{lemma}
\label{lem:fixbits}
  Consider a digraph $G$ having $G_k$ as a subgraph and an additional vertex $x \notin V(G_k)$.
  We use the compatibility function $\fthresh{1}$ and seek to find a $(k,1)$-representation for the
  induced subgraph $G[V(G_k) \cup \{x\}]$.
  Let $U \subseteq V(G_k^1)$ having that property that if $v \in V(G_k^1)$ 
  with $\ConflictSet_d(v) \subseteq \cup_{u \in U} \ConflictSet_d(u)$, then $v \in U$.
  Let $U' = \{u \in V(G_k^1) \suchthat u+1 \in U\}$.
  Let $Q = \cup_{u \in U} \ConflictSet_d(u)$.
%

  If $N_p(x) = V(G_1^k) \setminus U$, then $\ConflictSet_p(x) = Q$.
  If $N_d(x) = V(G_1^k) \setminus U'$, then $\ConflictSet_d(x) = Q$.
\end{lemma}
\begin{proof}
  We use the fact that there are exactly two bits set to one in the donor and patient vectors of each vertex in $G_k$
  in any $(k,1)$-representation.
  For the first statement, since $x$ has no edge from $u \in U$, $\ConflictSet_p(x) \supseteq \ConflictSet_d(u)$.
  Thus $\ConflictSet_p(x) \supseteq Q$.
  Now let $v \in V(G_k^1) \setminus U$ and $q_v \in \ConflictSet_d(v) \setminus Q$.
  If $q_v \in \ConflictSet_p(x)$, then for each $q \in Q$, there exists a vertex $w$ in $G_k^1$ with $\ConflictSet_d(w) = \{q,q_v\}$,
  so that $\{w,x\}$ would also not be an edge of $G$, a contradiction.
  Hence, $\ConflictSet_p(x) = Q$.
  The second statement follows analogously.
\end{proof}

\begingroup
\def\thetheorem{\ref{thm:np-hardness-main-paper}}
\begin{theorem}
  The {\sc $(k,t)$-representation with Ignored Edges} problem is $\mathcal{NP}$-complete.
\end{theorem}
\addtocounter{theorem}{-1}
\endgroup
\begin{proof}

  Consider a 3SAT formula on $n$ variables and with $m$ clauses.
  Set $k = 2n+2$, and
  build the following graph on $2 + n + m + \binom{k}{2} + k$ vertices.
  The first two vertices are labeled $v$ and $u$.
  Then there is 
  a vertex $v_i$ for each variable $i \in [n]$,
  a vertex $c$ for each clause $c \in [m]$.
  Call the subgraph induced by these $2 + n + m$ vertices $G'$.
  The last vertices come from the gadget $G_k$.
  
  The vertices in $G_k^2$ ground the $k$ bits used in each donor and patient vector..
  We think of the $k$ bits, in order, as corresponding to the $n$ positive literals, then their $n$ negations, followed by two ``extra'' bits.
  Then the index of literal $x_i$ will be $i$, and the index of literal $\bar x_i$ will be $n+i$.
  For $i$ and $j$ distinct in $V(G_k^2)$, $\abs{N_p(i) \cap N_p(j)} = 1$ within $G_k$.
  Denote this vertex of $G_k^1$ by $v(i,j)$,
  and without loss of generality we can assume that $\ConflictSet_d(v(i,j)) = \{i,j\}$.
  
  The edges among vertices in the induced subgraph $G_k$ are already defined; we define (a subset) of the rest of the edges.
  Together, these comprise precisely the subset $F$ of the edges and non-edges specified as an input the instance we are creating of
  {\sc $(k,t)$-representation with Ignored Edges}.

  Vertex $v$ has no incoming edges, and the only outgoing edges from $v$ to $V(G')$ are to every variable vertex $v_i$, $i \in [n]$.
  The rest of the vertices that are not in $G_k$ have no outgoing edges at all, to either $V(G')$ or $V(G_k)$, 
  and the only incoming edges are from vertices of $G_k^1$.  
  Vertex $u$ has an incoming edge from every vertex of $G_k^1$ except $v(2n+1,2n+2)$.
  For each variable vertex $v_i$, $i \in [n]$, it has an incoming edge from every vertex in $V(G_k^1)$ except $v(i,n+i)$.
  For each clause $c \in [m]$, let $\{c_1,c_2,c_3\}$ be the indices of the three literals that appear in $c$.
  Let $C \subset V(G_k^1)$ be $\{v(c_1,c_2), v(c_1,c_3), v(c_2,c_3), v(c_1,k), v(c_2,k) ,v(c_3,k)\}$.
  Then the vertex corresponding to $c$ has an incoming edge from every vertex in $V(G_k^1) \setminus C$.
  
  Every vertex of $V(G')$ except for $v$ will have a donor vector with every bit set to one 
  because there are no outgoing edges to any vertex of $G_k^1$,
  and $v$ will have an all-ones patient vector because it has no incoming edges from $G_k^1$.
  By Lemma~\ref{lem:fixbits}, in any $(k,1)$-representation of $G$,
  vertex $u$ will have $\ConflictSet_p(u) = \{2n+1,2n+2\}$.
  Variable vertex $v_i$, $i \in [n]$, will have $\ConflictSet_p(v_i) = \{i,n+i\}$.
  Clause vertex $c \in [m]$ will have $\ConflictSet_p(c) = \{c_1,c_2,c_3,2n+2\}$.
  
  Since the graph does not have an edge from $v$ to $u$, $\{2n+1, 2n+2\} \subseteq \ConflictSet_d(v)$
  (these are the only two conflict bits in $\ConflictSet_p(u)$ and the threshold is $1$). 
  Since the graph has an edge from $v$ to each variable vertex $v_i$, $i \in [n]$, 
  $\ConflictSet_d(v)$ must contain at most one of the indices corresponding to the variable or its negation 
  (there are no conflicts from the extra bits, which are set to $0$ in the patient vector of $v_i$). 
  Since the graph does not have an edge from $v$ to any of the clause vertices, 
  it has to have at least one conflict bit in a position corresponding to one of the three literals in the clause 
  (the other conflict comes from the extra bit $2n+2$).
  
  Thus, finding a suitable $(k,1)$-representation that satisfies the adjacencies of edges that appear in $F$
  would involve finding an appropriate set $\ConflictSet_d(v)$,
  which we have shown corresponds to choosing at most one value for each $x_i$,
  as well as choosing at least one literal that appears in each clause.
  This is the same as the problem of finding a satisfying formula for the initial instance of 3SAT.

\begin{figure}[t]
  \begin{center}
  \begin{tikzpicture}
  \tikzstyle{ns} = [draw,shape=circle,outer sep=0,inner sep=1,minimum size=12,font=\small]
  \tikzstyle{nsquare} = [draw,shape=rectangle,outer sep=0,inner sep=1,minimum size=12,minimum width=1.5cm,font=\small]
  
  \node [ns] (v) at (-2,0) {$v$};
  \node [ns] (c) at (0,4) {$c$};
  \node [ns] (u) at (0,-4) {$u$};
  \node [ns] (v1) at (0,2) {$v_1$};
  \node [ns] (v2) at (0,0) {$v_2$};
  \node [ns] (v3) at (0,-2) {$v_3$};
  \node [nsquare,label={[label distance=0cm,align=left]0:{$10100000$}}] (13) at (4,4.5) {$v(x_1,x_3)$};
  \node [nsquare,label={[label distance=0cm,align=left]0:{$10001000$}}] (15) at (4,3.5) {$v(x_1,\bar x_2)$};
  \node [nsquare,label={[label distance=0cm,align=left]0:{$00101000$}}] (35) at (4,2.5) {$v(x_3,\bar x_2)$};
  \node [nsquare,label={[label distance=0cm,align=left]0:{$10000001$}}] (18) at (4,1.5) {$v(x_1,8)$};
  \node [nsquare,label={[label distance=0cm,align=left]0:{$00100001$}}] (38) at (4,0.5) {$v(x_3,8)$};
  \node [nsquare,label={[label distance=0cm,align=left]0:{$00001001$}}] (58) at (4,-0.5) {$v(\bar x_2,8)$};
  \node [nsquare,label={[label distance=0cm,align=left]0:{$10010000$}}] (14) at (4,-1.5) {$v(x_1,\bar x_1)$};
  \node [nsquare,label={[label distance=0cm,align=left]0:{$01001000$}}] (25) at (4,-2.5) {$v(x_2,\bar x_2)$};
  \node [nsquare,label={[label distance=0cm,align=left]0:{$00100100$}}] (36) at (4,-3.5) {$v(x_3,\bar x_3)$};
  \node [nsquare,label={[label distance=0cm,align=left]0:{$00000011$}}] (78) at (4,-4.5) {$v(7,8)$};
  \draw  (v) edge[->] node {} (v1);
  \draw  (v) edge[->] node {} (v2);
  \draw  (v) edge[->] node {} (v3);
  \draw  (v) edge[->,dotted] node {} (c);
  \draw  (v) edge[->,dotted] node {} (u);
  \draw  (v1) edge[<-,dotted] node {} (14.west);
  \draw  (v2) edge[<-,dotted] node {} (25.west);
  \draw  (v3) edge[<-,dotted] node {} (36.west);
  \draw  (u) edge[<-,dotted] node {} (78.west);
  \draw  (c) edge[<-,dotted] node {} (13.west);
  \draw  (c) edge[<-,dotted] node {} (15.west);
  \draw  (c) edge[<-,dotted] node {} (35.west);
  \draw  (c) edge[<-,dotted] node {} (18.west);
  \draw  (c) edge[<-,dotted] node {} (38.west);
  \draw  (c) edge[<-,dotted] node {} (58.west);
  \end{tikzpicture}
  \end{center}
  \caption{Example of 3SAT reduction to $(k,t)$-representation.}
  \label{fig:reduction}
  \end{figure}
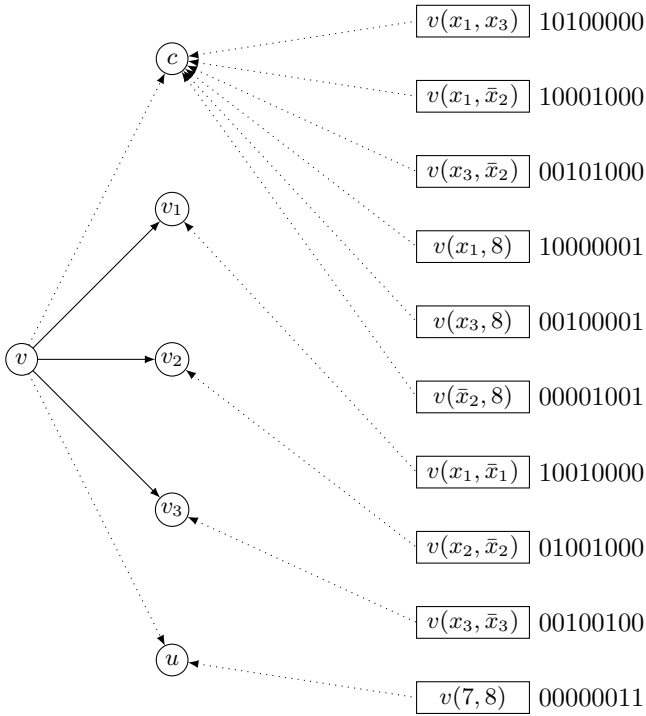
  
  As an example, consider the 3SAT formula $x_1 \lor \bar x_2 \lor x_3$.
  Figure~\ref{fig:reduction} shows the most relevant part of the graph used in the reduction.
  One possible $(k,1)$-representation may have $\ConflictSet_d(v) = \{1,7,8\}$, indicating $x_1 = 1$ and the rest of the variables are arbitrary.
  Another example of a possible representation is $\ConflictSet_d(v) = \{1,3,5,7,8\}$, meaning $x_1 = 1$, $x_2 = 0$ (index $5$ appears), and $x_3 = 1$.
\end{proof}

\end{document}